\documentclass[runningheads]{llncs}

\usepackage{eccv}

\usepackage{eccvabbrv}

\usepackage{graphicx}
\usepackage{booktabs}
\usepackage{dsfont}
\usepackage{multirow}
\usepackage{subcaption}

\usepackage[accsupp]{axessibility}  % Improves PDF readability for those with disabilities.

\usepackage{hyperref}

\usepackage{orcidlink}

\begin{document}

\title{RandMark: On Random Watermarking of Visual Foundation Models} 


\author{Anna Chistyakova \inst{1} \and
Mikhail Pautov\inst{1,2}}

\authorrunning{Anna Chistyakova and Mikhail Pautov}

\institute{Trusted AI Research Center, RAS\\ \and
AXXX\\
}

\maketitle

\begin{abstract}
    Being trained on large and diverse datasets, visual foundation models (VFMs) can be fine-tuned to achieve remarkable performance and efficiency in various downstream computer vision tasks. The high computational cost of data collection and training makes these models valuable assets, which motivates some VFM owners to distribute them alongside a license to protect their intellectual property rights. In this paper, we propose an approach to ownership verification of visual foundation models that leverages a small encoder-decoder network to embed digital watermarks into an internal representation of a hold-out set of input images. The method is based on random watermark embedding, which makes the watermark statistics detectable in functional copies of the watermarked model. Both theoretically and experimentally, we demonstrate that the proposed method yields a low probability of false detection for non-watermarked models and a low probability of false misdetection for watermarked models.
 
\end{abstract}
  %In recent years, a variety of ownership verification methods that help to control an illegal redistribution of neural networks has been proposed. 

%   the owners of some VFMs to distribute them alongside the license to protect their intellectual property rights. However, a dishonest user of the protected model's copy may illegally redistribute it, for example, to make a profit. As a consequence, the development of reliable ownership verification tools is of great importance today, since such methods can be used to differentiate between a redistributed copy of the protected model and an independent model. 

\section{Introduction}
\label{sec:intro}

Today, foundation models are deployed in different fields, for example, in natural language processing \cite{radford2019language,brown2020language}, computer vision \cite{ramesh2022hierarchical}, and biology \cite{ma2023towards}. Their impressive performance in a wide range of downstream tasks comes at a price of high cost of data collection, training, and maintenance. Consequently, the models become valuable assets of their owners: the user's access to foundation models is mainly organized via subscription to a service where the model is deployed or via purchasing the license to use a specific instance of the model. Unfortunately, some users may violate the terms of use (for example, by integrating their instances of the models into other services to make a profit). Hence, it is reasonable that the models’ owners are willing to defend their intellectual property from unauthorized usage by third parties. 

One of the prominent approaches to protecting the intellectual property rights (IPRs) of models is watermarking \cite{uchida2017embedding,guo2018watermarking,li2022defending}, the set of methods that embed specific information into a model by modifying its parameters. In watermarking, ownership verification is performed by checking for the presence of this information in a model. An alternative set of methods for IPR protection is based on fingerprinting, which typically does not alter the original model \cite{pautov2024probabilistically,quan2023fingerprinting,he2019sensitive}. Instead, these methods generate a unique identifier, or fingerprint, for the model; ownership verification is then conducted by comparing the fingerprint of the original model with that of the suspicious model.

This work introduces a method for watermarking visual foundation models (VFMs) by embedding digital watermarks into the hidden representations of a specific set of input images. Within the framework, we experimentally verify that embedding a watermark into the representation allows us to protect the ownership of VFMs fine-tuned for different practical tasks, such as image classification and segmentation. We demonstrate that our approach is able to distinguish between an independent model and functional copies of the watermarked model with high probability. 

Our contributions are summarized as follows:

\begin{itemize} 
    \item We propose RandMark, a novel methodology for watermarking visual foundation models. Unlike prior art focused on classifiers, our approach embeds binary signatures directly into the model's hidden representations via a set of trigger images, making it suitable for the diverse downstream use-cases of VFMs.
    % \item Experimental results indicate that the method can differentiate between a watermarked VFM and independently trained models, even after the original model has been fine-tuned for the downstream tasks such as classification and segmentation.
    \item We theoretically derive an upper bound on the probabilities of false positive detection of a non-watermarked model and misdetection of a functional copy of the watermarked model.% successful detection of functional copy of the watermarked model under mild assumptions on distributions of functional copies and independent models. 
    \item Through experiments on state-of-the-art visual foundation models (CLIP and DINOv2), we demonstrate that RandMark is highly robust. It successfully detects model ownership after various functional perturbations, including fine-tuning on downstream tasks (classification and segmentation) and unstructured pruning, where existing fingerprinting methods fail. % To our knowledge, this is the first work that addresses the problem of watermarking of visual foundation models.
\end{itemize}

\begin{figure}[t]
    \centering
    \includegraphics[width=\textwidth]{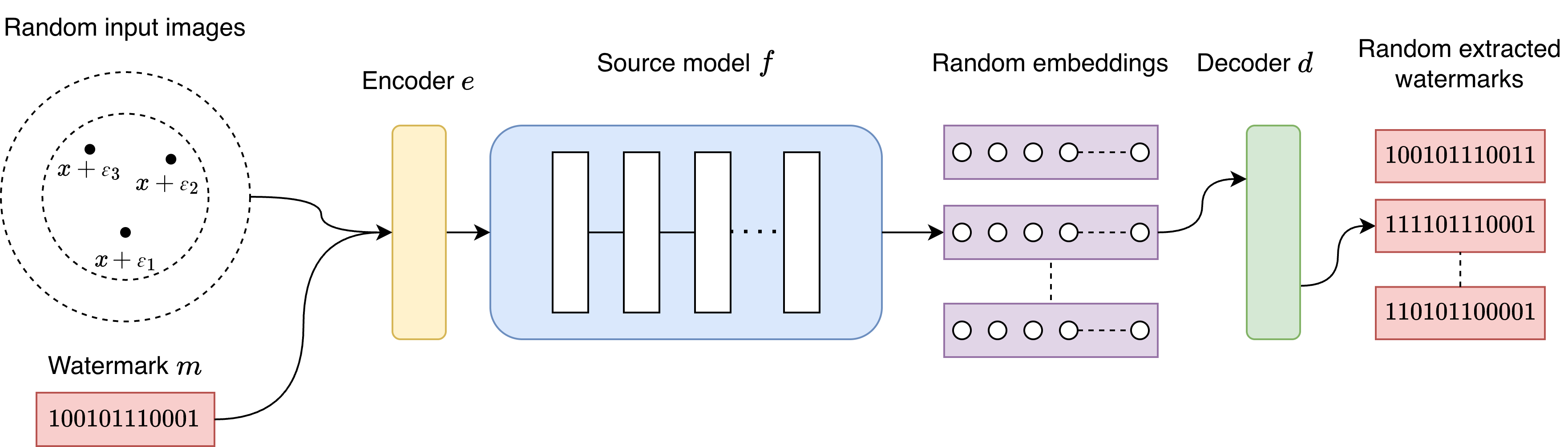}
    \caption{Overview of the proposed RandMark watermarking pipeline. A binary message is embedded into a visual foundation model using a set of trigger images and an encoder. During verification, randomized input transformations are applied to the trigger set, and a decoder extracts the watermark message from the model outputs. The extracted messages are then compared with the original watermark to verify model ownership.}
    \label{fig:rand_method}
\end{figure}

% \begin{figure}
%     \begin{center}
%         \includegraphics[width=0.75\textwidth]{images/express_print_num.pdf}
%         \caption{Schematic illustration of the proposed method. To embed the watermark, we pass an image $x$ to obtain its latent representation $p(x)$. Then, the first channel of $p(x)$, namely, $p_1(x)$ is concatenated with the watermark $m$ and passed to the encoder that produces the vector $e(\texttt{concat}(p_1(x),m))$. Later, the vector $e$ replaces the first channel of the original internal image representation, $p(x),$ and the updated vector $\tilde{p}(x)$ is passed to the latter part of VFM. To extract the watermark, we use a decoding network $d$ that maps an output of the VFM in the form $u = q(\tilde{p}(x))$ to the binary message $m',$ where $m'_i = \mathds{1}(d(u)_i \ge 1/2).$ Both encoder and decoder are represented by two fully connected layers.} 
%         \label{fig:method}
%     \end{center}
% \end{figure}

\section{Related work}

\subsection{Visual Foundation Models}

Visual foundation models, particularly those using vision transformers (ViT, \cite{DBLP:conf/iclr/DosovitskiyB0WZ21}), are widely used in modern computer vision due to their scalability and transferability across tasks. The advancement of self-supervised learning methods \cite{balestriero2023cookbook} has facilitated the creation of general-purpose models, including SimCLR \cite{chen2020simple}, DINO \cite{caron2021emerging}, CLIP \cite{radford2021learning}, and DINOv2 \cite{oquabdinov2}. These models learn representations from unlabeled images and demonstrate broad applicability across diverse tasks, often requiring minimal labeled data for fine-tuning. %Nevertheless, their internal mechanisms, specifically the characteristics of neural activations, remain an area of ongoing research.

% A recent concept in the analysis of these models is massive activations \cite{sun2024massive} -- unusually high response values in specific layers or tokens that contribute significantly to the model's output. These activations have been observed to occur across the latter layers of a model, are often of consistently high magnitudes, and can be located at the same spatial or token positions across different input images.

\subsection{Protecting Intellectual Property of Neural Networks}
% The use of watermarking and fingerprinting techniques to protect the intellectual property of neural networks has received increased attention within the field of trustworthy artificial intelligence. Research in this area includes various approaches: for instance, \cite{xu2024instructional} applies instruction tuning to fingerprint large language models, using a predefined private key to trigger a specific text output. In \cite{song2024manifpt}, the definitions of artifact and fingerprint in large generative models are formalized based on the geometric properties of the training data manifold. Another approach, proposed in \cite{pautov2024probabilistically}, involves using artificially generated images for the attribution of image classifiers under model extraction attacks. Furthermore, research in \cite{sander2024watermarking} indicates that it is possible to detect if a model was trained on the synthetic output of a watermarked large language model, highlighting a potential privacy consideration associated with neural network watermarking.

The protection of intellectual property for visual foundation models (VFMs) has gained increasing attention within the field of trustworthy AI. Watermarking and fingerprinting techniques aim to verify model ownership and prevent unauthorized usage or model extraction. While early works focused on large language models \cite{xu2024instructional, sander2024watermarking}, recent efforts adapt these ideas to visual models, including image classifiers and foundation models \cite{song2024manifpt, pautov2024probabilistically}. 

For visual foundation models (VFMs), there are currently no watermarking approaches specifically designed for these architectures. Several existing model ownership verification methods, such as ADV-TRA \cite{xu2024united}, and IPGuard \cite{cao2021ipguard}, have been proposed in the context of image classification. These approaches embed ownership signatures by either modifying the training objective or introducing crafted input patterns and then detect them based on the model’s responses. While these methods are effective for classification models, they are not directly tailored to the broader capabilities of VFMs, such as image feature extraction or downstream adaptation. Adapting watermarking techniques to visual foundation models thus remains an open challenge and motivates the work presented in this paper. Other complementary methods exploit weight-space smoothing or perturbations to embed ownership information directly into model parameters. For example, Bansal et al. \cite{bansal22a, ren23c} propose model watermarking through weight smoothing in deep neural networks, making the watermark robust to fine-tuning or minor architectural changes. These approaches provide alternative mechanisms to mark models without relying on specific input-output triggers and are particularly relevant for large visual foundation models where modifying the backbone is costly.  

Overall, while watermarking for VFMs is still in early stages, these methods illustrate that both data-driven triggers and weight-space techniques can serve as practical IP protection strategies for high-capacity visual models.

\section{Methodology}
\subsection{Problem Statement}

In this work, we focus on the problem of watermarking of visual foundation models. To describe the proposed method, we start by introducing the notations. Let $s$ be the dimension of the input image and $f:\mathbb{R}^s\to \mathbb{R}^k$ be the source VFM that maps input images to the embeddings of dimension $k$. Here and below, we will write $f' \sim f$ to indicate that the model $f'$ is a functional copy of $f$ that is obtained, for example, via fine-tuning, knowledge distillation or pruning of the original model. Analogously, by writing $g \perp f$ we will indicate that two models, $g$ and $f$, are independent of each other. In our method, we train two auxiliary models, the encoder $e:\mathbb{R}^s \times \{0,1\}^n \to \mathbb{R}^k$ that embeds the binary message $m$ of length $n$ into the representation of the input object $x \in \mathbb{R}^s$, and the decoder $d:\mathbb{R}^k\to \{0,1\}^n$ that extracts a binary message from the output embedding of the VFM. Given the input image $x$, the source model $f$ and the message $m$ embedded into $f(x)$, the goal of the method is two-fold: on the one hand, the decoder $d$ should extract close messages from the representations $f(x)$ and $f'(x)$ for the model $f' \sim f$; on the other hand, given the model $g \perp f$, the messages extracted from the representations $f(x)$ and $g(x)$ have to be far apart.

% $\Omega_f$ be the space of visual foundation models that are functionally dependent on $f$, and $\Xi$ be the space of functionally independent visual foundation models. Let $h$ be the dimension of the hidden representation of the image and $m\in \{0,1\}^n$ be the binary vector of length $n$.   In addition, let $f$ be the composition $f(x) \equiv q(p(x)),$ where $p:\mathbb{R}^s\to\mathbb{R}^h$ maps an image to the hidden representation, and $q:\mathbb{R}^h\to\mathbb{R}^d$ maps the vector from the  hidden representation to the output embedding of the VFM. 

%Additionally, we fine-tune the latter part of the source model, namely, $q$. 

The formal problem statement goes as follows. Given $x$ as the secret input image used for watermarking, a predefined threshold $\tau \ll n$ and probability thresholds $ 0 < \gamma_1 \ll \gamma_2 < 1$, the following inequalities should hold:
\begin{equation}
\label{eq:statement}
    \begin{cases}
      \displaystyle 
        \mathbb{P} \left(\|w - d(f'(x))\|_1 \le \tau\right) \ge \gamma_1, \\ 
        \\
        \mathbb{P} \left(\|w - d(g(x))\|_1 \le \tau\right) \le \gamma_2, \\
    \end{cases}
 \end{equation}
where $w = e(x,m)$ is the embedding with the watermark, $f'\sim f, \ g \perp f.$ In Eq.~\ref{eq:statement}, the probabilities are taken over the randomness induced by the encoder; this randomness will be discussed in the subsequent sections.

% In this work, we present a method to watermark visual foundation models by training an auxiliary network that embeds binary messages into hidden representations of input images of the source VFM. We start by introducing the notations used throughout the paper. Namely, let s be the dimension of an image, Ω be the space of visual foundation models, f:Rs →Rd be the source VFM that maps input images to embeddings of dimension d, h be the dimension of hidden image representation, and let m {0,1}k be the binary vector of length k. Let f be the composition f(x) ≡ q(p(x)), where p: Rs→ Rh and q:Rh → Rd represent mappings from an image to the hidden representation and from the hidden representation to the output embeddings, respectively. In our method, we train two auxiliary models, namely, encoder e:Rh{0,1}k → Rh that embeds the binary message m into the hidden representation p(x) and decoder d: Rd → {0,1}k that extracts binary messages from q(x). In addition, we fine-tune the latter part of the source model, namely, q. Given the image x, the message m embedded into p(x) and transform π:Ω→Ω that maps the foundation model to its functional copy,  the goal of the method is two-fold: on the one hand, the decoder d should extract close messages from hidden representations of f and π(f); on the other hand, given the model g which is functionally independent of f, the messages extracted from hidden representations of f and g should be far apart. 

\subsection{Threat Model}
In this section, we discuss the conditions under which the proposed method is expected to operate correctly and outline the potential adversary’s capabilities. The goal of an adversary is to remove an existing watermark from a model so that ownership cannot be verified. Specifically, an adversary may attempt either a watermark removal attack, aiming to eliminate the watermark while preserving the model’s functionality, or a model extraction attack, trying to obtain a copy of the watermarked model without the watermark. Possible attacks include fine-tuning the model on downstream tasks or pruning. The objective of the watermarking method is to reliably determine whether a suspect model is a functional copy of the watermarked visual foundation model.

\subsection{Proposed Method}
We introduce RandMark, a novel watermarking approach designed for visual foundation models. RandMark embeds user-specific binary signatures into the representations of a randomly transformed set of input images. To do so, we fine-tune the source model together with the lightweight encoder and decoder networks. This approach enables ownership verification by extracting digital fingerprints from the set of randomly transformed specific set of input images and computing the statistic of resulting random variables. 

The watermarking process goes as follows. First of all, given input image $x$ and user-specific binary message $m,$ we inject $m$ into the representation of $x + \varepsilon_j, \ \varepsilon_j \sim \mathcal{N}(0, \sigma^2(x) I)$ by training the small encoder $e$ and fine-tuning the source model $f$. Modified representation, $f(x+\varepsilon_j)$, is then passed to the decoder network $d$ that extracts binary message $m'_j$ from it. We highlight that the extracted messages, $m'_j$, are random variables due to the randomness in transformation of input image. 
The encoder, decoder, and the source foundation model are trained jointly to minimize both the average discrepancy between $m$ and $m'_j$ and the variance of $m'_j.$

% \subsection{ActiveMark in a Nutshell}
% We introduce ActiveMark, a novel watermarking approach designed specifically for visual foundation models. ActiveMark  \ embeds user-specific binary signatures into a preselected internal feature representation of a holdout set of input images. To do so, we fine-tune a small number of latter layers of the source VFM together with training the lightweight encoder and decoder networks. This approach enables ownership verification by extracting digital fingerprints directly from the model's activations when provided with specific input images. Unlike traditional watermarking techniques that modify model weights or outputs, our method introduces minimal architectural changes while preserving the functional capacity of the model.

% The watermarking procedure goes as follows. First of all, given input image $x$ and user-specific binary signature $m$, we train a small encoder network $e$ that injects $m$ into a selected channel of the internal activation of the preselected transformer block of the source model $f$. Then, the modified representation of $x$, namely, $e(p(x), m)$, is propagated through the latter part of $f$, namely, through $q$. At the last block, the decoder network $d$ extracts the binary message $m'$ from the output embedding of $e(p(x), m)$, namely, $m' = d(q(e(p(x), m))).$

% The encoder, decoder, and the latter part of the foundation model are trained jointly to minimize the discrepancy between $m'$ and $m$ while leaving the output embedding intact. 

\subsubsection{Loss function} The training objective is the combination of two terms: given the input sample $x$, the first one ensures that the feature representations of the watermarked and original models do not deviate much; the second term forces the extracted binary messages to be close to the embedded one. Specifically, the objective function is 
\begin{equation}
    L(x, f, \tilde{f}) = \|f(x) - \tilde{f}(x)\|_2 +\frac{\lambda}{K} \sum_{j=1}^K \|m - m_j'\|_2,
\end{equation}
where $\lambda >0$ is a scalar parameter, $\tilde{f}$ is the watermarked version of $f$ and $m_j' = d(\tilde{f}(e(x+\varepsilon_j, m)))$ is the binary message extracted by the decoder from $x + \varepsilon_j$ and $K$ is the total number of transformations of the input image. This formulation ensures the successful embedding and extraction of watermarks with little to no impact on the feature representation.

% \begin{figure}
%     \begin{center}
%         \includegraphics[width=\textwidth]{images/massive_activations_old.pdf}
%         \caption{The average magnitudes of activations of blocks of the source VFMs. It is noteworthy that starting from a particular block, the magnitudes of activations increase drastically, namely, starting from block $18$ of DINOv2 and from block $12$ of CLIP.} 
%         \label{fig:activations}
%     \end{center}
% \end{figure}

% This motivates our selection of these blocks as carriers for signature embedding. For each selected expressive block, we further localize the most influential tokens — the ones that are most critical for the block's output. For each token, we compute the absolute values of output activations and choose the token as an outlier if the corresponding output activation increases drastically after some block (namely, if on a particular layer its z-score increases up to threshold value $t$). We propose to use these outlier tokens as internal activation anchors for binary message injection.

\subsubsection{Evaluating the efficiency of the method}

To evaluate the performance of the proposed method, given the user-specific watermark $m$ and input image $x$, we compute both the sample average and the sample variance of the variable $\|m - m'\|_1$, where $m'$ is the  watermark. Here we recall that the extracted watermarks are random.
% Note that we specifically indicate that the extracted message depends both on the input image and the model from which it is extracted. 

Namely, if the total number of transformations of the input image is $K$ and the length of the watermark is $n$,
we measure the average number of matching bits between $m$ and $m'_j$ in the form 
\begin{equation}
\label{eq:mean}
    \hat{\mathbb{E}}\|m - m'\|_1 = \frac{1}{K} \sum_{j=1}^K \sum_{i=1}^n  \mathds{1}\left(m_i \ne ( m'_j)_i\right),
\end{equation}
and the sample variance is computed as
\begin{equation}
\label{eq:var}
    \hat{\mathbb{V}}\|m-m'\|_1 = \frac{1}{K-1}\sum_{j=1}^K \left(d_j - \hat{\mathbb{E}}\|m-m'\|_1\right)^2,
\end{equation}
where $d_j = \sum_{i=1}^n \mathds{1}\left(m_i \ne (m'_j)_i\right),$ and $m_j' = d(\tilde{f}(e(x+\varepsilon_j, m))).$ The intuition behind using this two metrics is as follows. First of all, given the extracted message $m',$ the distance from Eq. \eqref{eq:mean} is expected to be small for the watermarked model and large for an independent model. Secondly, if we introduce an auxiliary variable in the form 
\begin{align}
    v(f, h) = \mathbb{V}\left(\|m'(f) - m'(h)\|_1\right), 
\end{align}
then $v(f,f')$ is expected to be small for $f' \sim f$ and  $v(f,g)$ is expected to be large for $g \perp f$. We elaborate on this point in the subsequent sections.

% Recall that a good watermarking method has to satisfy two conditions: on the one hand, given  input image $x$, for the watermarked model $f$, the distance has to be close to $0$; on the other hand, for a separate (independent) model, the distance has to be close to $n$. 

In this work, the decision rule that is used to evaluate whether the given network is watermarked is the comparison of the distance with a predefined threshold: given the suspicious model $h$, input image $x$, secret message $m$ and the series of $K$ watermarks $m'_1, m'_2,\dots,m'_K$ extracted from $h$, we treat $h$ as watermarked if  
\begin{equation}
    \rho(x) = \hat{\mathbb{E}}\|m-m'\|_1 = \frac{1}{K}\sum_{j=1}^K \sum_{i=1}^n \mathds{1}\left(m_i \ne (m'_j)_i\right) \le \tau,
\end{equation}
where $\tau\ge0$ is the threshold value. In the case of many input images used for watermarking, namely, for $N$ images from $ \mathcal{X} = \{x_1,\dots,x_N\}$, the performance of the method is illustrated by the watermark detection rate, $R(h,\mathcal{X},\tau)$, in the form below:
\begin{equation}
\label{eq:R}
  R(h, \mathcal{X}, \tau) = \frac{1}{N}\sum_{x_i \in \mathcal{X}} \mathds{1}[\rho(x_i) \le \tau].
\end{equation}
As an auxiliary indicator of the model being watermarked, for each $x\in \mathcal{X},$ we compute the value of statistic $v(f,h).$ 
\subsubsection{Setting the threshold value}

We set the threshold by formulating a hypothesis test: the null hypothesis, $H_0 = $  ``the model $h$ is not watermarked'', is tested against an alternative hypothesis, $H_1 = $  ``the model $h$ is watermarked'', for the given suspicious model $h$. In this section, %we assume that the messages $m'(g,x)$ extracted from all the non-watermarked models $g$ are distributed uniformly over all bit strings of length $n$. Additionally, 
we assume that the probabilities that the $i'$th bit in $m'(f)$ and $m'(h)$ coincide are the same for all $i \in [1,n].$   Having said so, we estimate the probability of false acceptance of hypothesis $H_1$ (namely, $FPR_1$) as follows:
\begin{equation}
    FPR_1 = \mathbb{P}_{g \sim \Xi} [\rho(m, m'(g,x))  < \tau] = \sum_{j=0}^\tau \binom{n}{j} (1-r)^j r^{n-j}, 
\end{equation}
where $r = \mathbb{P}_{g \sim \Xi}(m_i = m'(g,x)_i).$
To choose a proper threshold value for $\tau$, we set up an upper bound for $FPR_1$ as $\varepsilon$ and solve for $\tau$, namely, 
\begin{equation}
    \tau = \arg\max_{\tau' < n}  \left(\sum_{j=0}^{\tau'} \binom{n}{j} (1-r)^j r^{n-j}\right) \quad \text{s.t.} \quad \sum_{j=0}^{\tau'} \binom{n}{j} (1-r)^j r^{n-j} < \varepsilon.
\end{equation}

\subsection{Difference between watermarked and non-watermarked models}
\label{sec:probs}
 
Recall that a good watermarking approach should yield a high watermark detection rate from \eqref{eq:R} for the models that are functionally connected to the watermarked one and, at the same time, low detection rates for independent models. To assess the integrity of the proposed approach, we estimate the probabilities of the method to yield low detection rates for functionally dependent models and high detection rates for independent models in the form 
\begin{equation}
\label{eq:detection_probs}
  \mathbb{P}_{f' \sim \Omega_f}[R(f', \mathcal{X}, \tau) < \overline{R}], \quad \mathbb{P}_{g \sim \Xi}[R(g, \mathcal{X}, \tau) > \underline{R}]
\end{equation}
for some threshold values $0 < \underline{R} < \overline{R} < N.$ 

To estimate the probabilities from \eqref{eq:detection_probs}, we firstly provide one-sided interval estimations for conditional probabilities of bit collisions in the form
\begin{equation}
\label{eq:bits_probabilities}
    r(\Omega_f|x) = \mathbb{P}_{f' \sim \Omega_f} [m_i = m'(f',x)_i], \quad r(\Xi|x) = \mathbb{P}_{g \sim \Xi} [m_i = m'(g,x)_i].
\end{equation}
We do it by sampling $M$ functionally dependent models, namely, $f_1', \dots, f'_M \sim \Omega_f,$ and $M$ independent models, namely $g_1, \dots, g_M \sim \Xi.$ Here, the space $\Xi$ of independent models consists of visual foundation models, both of the same architecture and of different architectures as $f$, by either fine-tuning of \emph{non-watermarked} copy of $f$ for a downstream task, of via functionality stealing perturbations, for example, via knowledge distillation \cite{hinton2014distilling} or pruning \cite{han2016deep}. Similarly, the space $\Omega_f$ consists of the models, both of the same architecture and of different architectures as $f$, by either fine-tuning of $f$ for a downstream task, of via functionality stealing perturbations. 

Then, given the set $\mathcal{X} = \{x_1, \dots, x_N\}$ of images used for the watermarking of $f$ from \eqref{eq:R}, we compute the quantities
 \begin{equation}
     \mathds{1}(f'_j,i, x_l) = \mathds{1}[m_i = m'(f'_j, x_l)_i] \quad \text{and} \quad \mathds{1}(g_j,i, x_l) = \mathds{1}[m_i = m'(g_j, x_l)_i]
 \end{equation}
%  as the samples to estimate the probabilities from \eqref{eq:bits_probabilities} and compute one-sided Clopper-Pearson confidence intervals for $r(\Omega_f|x)$ and $r(\Xi|x)$ in the form 
 \begin{equation}
 \label{eq:bit_estimates}
  \begin{cases}
    \mathbb{P}(r(\Omega_f|x) < l(x)) \le \frac{\alpha}{N},\\
    \mathbb{P}(r(\Xi|x) > u(x)) \le \frac{\alpha}{N}.
\end{cases} 
\end{equation}
These estimates, namely, $l(x)$ and $u(x)$, are used to estimate the probabilities from \eqref{eq:detection_probs}.
 
%  \begin{equation}
%      \mathbb{E}_{f' \sim \Omega_f} (\mathds{1}(f', i,x_l)) = r(\Omega_f|x_l), \quad  \mathbb{E}_{g\sim \Xi}(\mathds{1}(g, i,x_l)) = r(\Xi|x_l).
%  \end{equation}
% Now we can bound the probabilities from \eqref{eq:statement} using the statistics from \eqref{eq:bits_probabilities}:
% \begin{align}
% \label{eq:upper_bound}
%     & \mathbb{P}_{f' \sim \Omega_f} (\rho(m, m'(f',x))  < \tau ) > \sum_{j=0}^\tau \binom{n}{j} (1 - r(\Omega_f))^j r(\Omega_f)^{n-j} \equiv \gamma_1, \nonumber \\ 
%      &\mathbb{P}_{g \sim \Xi} (\rho(m, m'(g,x))  < \tau ) < \sum_{j=0}^\tau \binom{n}{j} (1 - r(\Xi))^j r(\Xi)^{n-j} \equiv \gamma_2.
% \end{align}
% In the next section, we describe an approach to estimate $r(\Omega_f)$ and $r(\Xi).$

\subsubsection{Estimating the probability of a deviation of the  detection rate}

%Without loss of generality, we discuss the estimation of the probability $\mathbb{P}_{f' \sim \Omega_f}[R(f', \mathcal{X}, \tau) < \overline{R}]$ computed over functionally dependent models. 
In this section, we discuss how to upper-bound both the probability of false detection of a non-watermarked model as a copy of the watermarked one and the probability of misdetecting a functional copy of the watermarked model.  

Note that $R(f', \mathcal{X}, \tau)$ is a sum of $N$ independent Bernoulli variables with parameters, $r(\Omega_f|x)$, so 
\begin{equation}
    \mathbb{P}_{f' \sim \Omega_f}[R(f', \mathcal{X}, \tau) < \overline{R}] = \sum_{l=0}^{\overline{R}-1}\sum_{S \subset \mathcal{X}: |S|=l} \prod_{x_{in}\in S} r(\Omega_f|x_{in}) \prod_{x_{out} \notin S} (1 - r(\Omega_f|x_{out})).
\end{equation}
Note that replacing the parameters $r(\Omega_f|x)$ with its estimations in the form $l(x)$ from \eqref{eq:bit_estimates} yields the bound
\begin{equation}
\label{eq:bound_for_dep}
    \mathbb{P}_{f' \sim \Omega_f}[R(f', \mathcal{X}, \tau) < \overline{R}] < \sum_{l=0}^{\overline{R}-1}\sum_{S \subset \mathcal{X}: |S|=l} \prod_{x_{in}\in S} l(x_{in}) \prod_{x_{out} \notin S} (1 - l(x_{out})) = \underline{p}(\Omega),
\end{equation}
%To prove the bound from \eqref{eq:bound_for_dep},  

%The bound from \eqref{eq:bound_for_dep}  
that holds with probability at least $1-\alpha.$  Similarly, 
\begin{equation}
    \mathbb{P}_{g \sim \Xi}[R(g, \mathcal{X}, \tau) > \underline{R}] < \sum_{l=\underline{R}+1}^{N}\sum_{S \subset \mathcal{X}: |S|=l} \prod_{x_{in}\in S} u(x_{in}) \prod_{x_{out} \notin S} (1 - u(x_{out})) = \underline{p}(\Xi).
\end{equation}
\begin{remark}
During experimentally, we used $n=32, \tau=5, M=1000$ and varied confidence level $\alpha$ such that probabilities $\alpha, \underline{p}(\Xi), \underline{p}(\Omega)$ were close. Specifically, value $\alpha = 5\times 10^{-6}$ yields $\underline{p}(\Omega) = 10^{-6}, \underline{p}(\Xi)=10^{-4}$ and  $\overline{R} = 750, \underline{R}=600$. 
\end{remark}
Thus, if one uses the boundary values $\underline{R}, \overline{R}$ to distinguish between the watermarked and non-watermarked model, one is guaranteed to have both error probabilities $\underline{p}(\Omega), \underline{p}(\Xi)$ low.

\subsection{Alternative estimation of bit collisions}

According to equation \ref{eq:R}, the quantity $R(f, \mathcal{X}, \tau) = \sum_{i=1}^N \mathds{1}[\rho(m(x_i), m'(f, x_i)) \le \tau]$ is the sum of $N$ independent Bernoulli random variables. We may rewrite $R_1 = R(f', \mathcal{X}, \tau)$ and $R_2 = R(g, \mathcal{X}, \tau)$ from equation \ref{eq:detection_probs} in the form
\begin{align}
    & R_1 = \xi_1 + \xi_2 + \dots + \xi_{n-1} + \xi_n, \nonumber \\
    & R_2 = \eta_1 + \eta_2 + \dots + \eta_{n-1} + \eta_n, 
\end{align}
where $\xi_i \sim \emph{Bernoulli}(p_i), \eta_i \sim \emph{Bernoulli}(q_i)$ are independent and parameters $(p_i, q_i)$ are unknown. Let  $\overline{p} = \frac{1}{n}\sum_{i=1}^n p_i$ and $\overline{q} =  \frac{1}{n}\sum_{i=1}^n q_i$. Then, if $\overline{R} < n \overline{p}$ and $\underline{R} >n\overline{q}$ from equation \ref{eq:detection_probs}, the following lemma holds.

\begin{lemma}
Let $\delta >0$ and set $\varepsilon = \sqrt{\frac{1}{2n}\ln \left(\frac{1}{\delta}\right)}$. Let $\hat{p} = \frac{1}{n} R_1$ and $\hat{q} = \frac{1}{n} R_2$ be unbiased estimates of $\overline{p}$ and $\overline{q},$ respectively. Then, with  probability at least $1-\delta,$ the following upper bounds for probabilities from equation \ref{eq:detection_probs} hold: 
\begin{align}
\label{eq:new_lemma}
    & \mathbb{P}_{f' \sim \Omega_f}[R(f', \mathcal{X}, \tau) < \overline{R}] \le h(\hat{p}, \varepsilon^{-}), \nonumber \\ 
    &  \mathbb{P}_{g \sim \Xi}[R(g, \mathcal{X}, \tau) > \underline{R}] \le h(\hat{q},\varepsilon^{+}),
\end{align}
 \text{where}
\begin{align}
    &h(\hat{p}, \varepsilon^{-}) = \left(\frac{n(\hat{p}-\varepsilon)}{\overline{R}}\right)^{\overline{R}} \left( \frac{n(1 - (\hat{p} - \varepsilon))}{n - \overline{R}}\right)^{n-\overline{R}}, \nonumber \\
     & h(\hat{p}, \varepsilon^{+}) = \left(\frac{n(\hat{p}+\varepsilon)}{\underline{R}}\right)^{\underline{R}} \left( \frac{n(1 - (\hat{p} + \varepsilon))}{n - \underline{R}}\right)^{n-\underline{R}}.
\end{align}
\end{lemma}

\begin{remark}
Some words about relation $\overline{R}$ and $\overline{p};$ proof will be moved to the appendix. 
\end{remark}

\begin{proof}
We provide a proof for the upper inequality from equation \ref{eq:new_lemma}. Specifically, we need to upper bound the probability $\mathbb{P}(R \le d),$ where $R \equiv R(f', \mathcal{X}, \tau)$ and $d\equiv \overline{R}.$ According to Chernoff bound,
\begin{equation}
    \mathbb{P}(R < d) \le \inf_{t<0} \exp(-td) \mathbb{E}(\exp(tR)).
\end{equation}
Note that, according to independence of $\xi_i,$ 
\begin{equation}
    \mathbb{E}(\exp(tR)) = \prod_{i=1}^n \mathbb{E}(\exp t \xi_i) = \prod_{i=1}^n (1-p_i + p_i e^t),  \ \text{and, hence,}
\end{equation}
\begin{equation}
     \mathbb{P}(R < d) \le \inf_{t<0} \left[\exp(-td)\prod_{i=1}^n (1-p_i +p_i e^t)\right].
\end{equation}

Let $\phi(p) = \ln (1-p + p e^t).$ Note that $\phi''(p) = -\frac{(e^t-1)^2}{(1-p+pe^t)^2}<0$ for all $t<0$, and, hence, $\phi(p)$ is strictly concave on $[0,1]$. 

From the concavity of $\phi(p),$ it follows that 
\begin{align}
    & \frac{1}{n}\sum_{i=1}^n \ln (1-p_i +p_ie^t) \le \ln (1-\overline{p} + \overline{p}e^t), \nonumber \\
    & \prod_{i=1}^n \left[1- p_i +p_i e^t\right] \le (1-\overline{p}+\overline{p}e^t)^n,
\end{align}
and, consequently,
\begin{equation}
    \mathbb{P}(R<d) \le \inf_{t<0} \left[\exp(-td) \left(1-\overline{p} + \overline{p}e^t \right)^n\right].
\end{equation}

Denote $\psi(t) = \exp(-td)(1-p+p e^t)^n.$ To find $\inf_{t<0} \psi(t),$ we analyze the derivatives of its logarithm:
\begin{equation}
    \frac{d}{dt} \ln \psi(t) = -d + \frac{n p e^t}{1-p+p e^t}
\end{equation}
Note that  $\frac{d}{dt} \ln \psi(t) = 0$ iff $-d + \frac{np y}{1-p+py} = 0,$ where $y =e^t < 1.$ 
\begin{equation}
   0 = -d + \frac{npy}{1-p+py}  \leftrightarrow y = \frac{d-pd}{np - pd}  \leftrightarrow t = \ln \left(\frac{d-pd}{np-pd}\right).
\end{equation}
To satisfy $y<1,$ it is required that $d < np.$
Since $\frac{d^2}{dt^2} \ln \psi(t) = \frac{np e^t (1-p)}{(1-p+pe^t)^2}>0$ is monotonic, $t = \ln \left(\frac{d-pd}{np-pd}\right)$ is a unique critical point of $\frac{d}{dt} \ln \psi(t).$
Thus, 
\begin{equation}
    \inf_{t<0} \psi(t) = \gamma(p) = \left(\frac{np}{d}\right)^d \left(\frac{n(1-p)}{n-d}\right)^{n-d}
\end{equation}
and the overall bound is 
\begin{equation}
    \mathbb{P}(R<d) \le \gamma(\overline{p})  
\end{equation}
Note that  $\gamma(\overline{p})$ is impossible to compute directly (since $\overline{p}$) is unknown. Instead, we note that $\gamma(p)$ is monotonic in $p.$
Indeed, 
\begin{align}
& \frac{d}{dp} \ln \gamma(p) = \frac{d}{p} - \frac{n-d}{1-p} \rightarrow    \frac{d}{dp} \ln \gamma(p) = 0 \ \text{at } p =\frac{d}{n}, \\
& \frac{d^2}{dp^2} \ln \gamma(p) = -\frac{d}{p^2} - \frac{n-d}{(1-p)^2} < 0,
\end{align}
and hence $\gamma(p)$ has a unique global maximum at $p=\frac{d}{n},$ is strictly increasing on $[0, \frac{d}{n})$ and is strictly decreasing on $(\frac{d}{n}, 1].$ 
Recall that $d<np,$ so $\gamma(p)$ is decreasing in $p.$ 

If we set $\hat{p} = \frac{1}{n}R_1,$ then, according to Hoeffding inequality, 
\begin{equation}
    \mathbb{P}(\overline{p} \le \hat{p} - \varepsilon) \le \exp(-2\varepsilon^2 n)
\end{equation}
and from monotonicity of $\gamma(p)$ for $\hat{p}-\varepsilon > \frac{d}{n},$ 
\begin{equation}
    \mathbb{P}(R<d) \le \gamma(\hat{p}-\varepsilon) 
\end{equation}
with probability at least $1-\delta$ for $\varepsilon = \sqrt{\frac{1}{2n}\ln\left(\frac{1}{\delta}\right)},$ what concludes the proof. The proof of the second inequality from equation \ref{eq:new_lemma} is similar. 
% \textcolor{blue}{We verify the necessary assumptions $d < np$ and $\hat{p} - \varepsilon > \frac{d}{n}$ experimentally. }
\end{proof}

\section{Experiments}

We conducted our experiments using two large-scale VFMs, CLIP \cite{radford2021learning} and DINOv2 \cite{oquabdinov2}. To train models on downstream tasks (namely, for classification and segmentation), we utilized three domain-specific datasets: 
\begin{itemize}
    \item E-commerce Product Images: This dataset consists of $18,175$ product images categorized into $9$ major classes based on Amazon’s product taxonomy. It is primarily used for image-based product categorization.
    \item FoodSeg103: A food image segmentation dataset containing $7,118$ images annotated with fine-grained pixel-wise labels for over $100$ food categories. It supports both semantic segmentation and instance-level analysis of food items.
\end{itemize}

\subsection{Watermark injection}
% Both source VFM were initialized with publicly available pretrained weights. To embed the watermarks, we use a random subset of $N=1000$ images from the ImageNet dataset \cite{deng2009imagenet}. For each image, we randomly and uniformly sample a binary vector $m$ of size $n=32$ and assign it as the corresponding watermark for this image. Later, we concatenate the first feature channel of $p(x)$ with the message $m$ to obtain the vector $ v = \texttt{concat}(p_1(x), m)$ which is passed to the encoder $e.$ The output of the encoder, namely, $e(v) \in \mathbb{R}^h$ then replaces the first feature channel of  $p(x)$. Next, an updated representation $\tilde{p}(x)$ that differs from the original $p(x)$ in the first channel is passed to the latter part of the VFM, namely, to $q.$ The output of the VFM in the form $ u = q(\tilde{p}(x))$ is passed to the lightweight decoder, $d$, that produces a vector $d(u)$ of size $n.$ An extracted binary message, $m',$ is binarization of $d(u)$ in the form $m'_i = \mathds{1}(d(u)_i \ge 1/2).$ The schematic illustration of the method is presented in Fig. \ref{fig:rand_method}.
Both source VFMs were initialized with publicly available pretrained weights. To embed watermarks, we use a random subset of $N=1000$ images from the ImageNet dataset \cite{deng2009imagenet}, assigning each image a randomly sampled binary vector $m$ of size $n=32$. The image and its corresponding message are jointly processed by an encoder, and the resulting output is forwarded through the VFM, allowing the watermark information to be incorporated while preserving model functionality. The embedded watermark can later be extracted using the RandMark\ procedure, producing a binary message $m'$ to verify model ownership. A schematic illustration of the method is presented in Fig.~\ref{fig:rand_method}, and detailed architectural descriptions are provided in the Supplementary Material.

\subsection{Functional perturbations of VFM}
To illustrate the robustness of watermarks embedded by RandMark, we evaluate how the detection rate from \eqref{eq:R} changes under fine-tuning of the model for downstream tasks and under pruning. Specifically, we fine-tune all the layers of the watermarked VFM for both classification and segmentation downstream tasks, using the aforementioned datasets. The fine-tuning was performed using the AdamW optimizer for $10$ epochs.

To investigate the impact of model sparsity on both classification accuracy and watermark robustness, we applied post-training unstructured $l_1$-norm pruning to the entire model. We evaluated two sparsity levels: moderate pruning, where 20\% of the lowest-magnitude weights were zeroed out, and aggressive pruning, where 40\% of the weights were removed. This procedure enabled us to assess the effect of varying sparsity levels on watermark reconstruction. Note that the  unrestricted $l_1$-norm pruning is used purely as the baseline to illustrate the robustness of the proposed method to the modifications of the model. 

\section{Results}
\begin{figure}[htb]
   \centering
        \subfloat[\centering The architecture of VFM is CLIP]{{\includegraphics[width=0.46\textwidth]{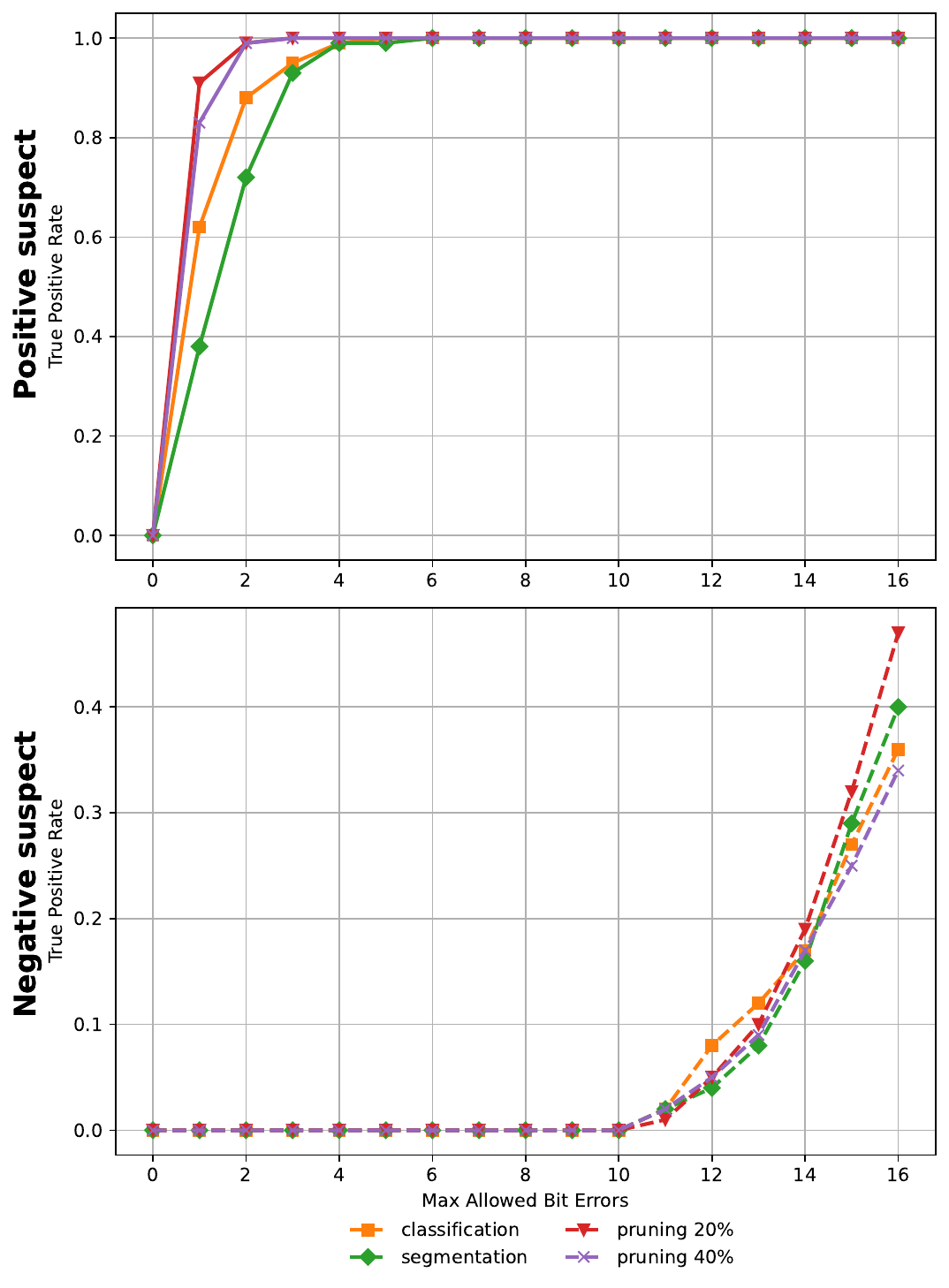} }}%
    \qquad
        \subfloat[\centering  The architecture of VFM is DINOv2]{{\includegraphics[width=0.46\textwidth]{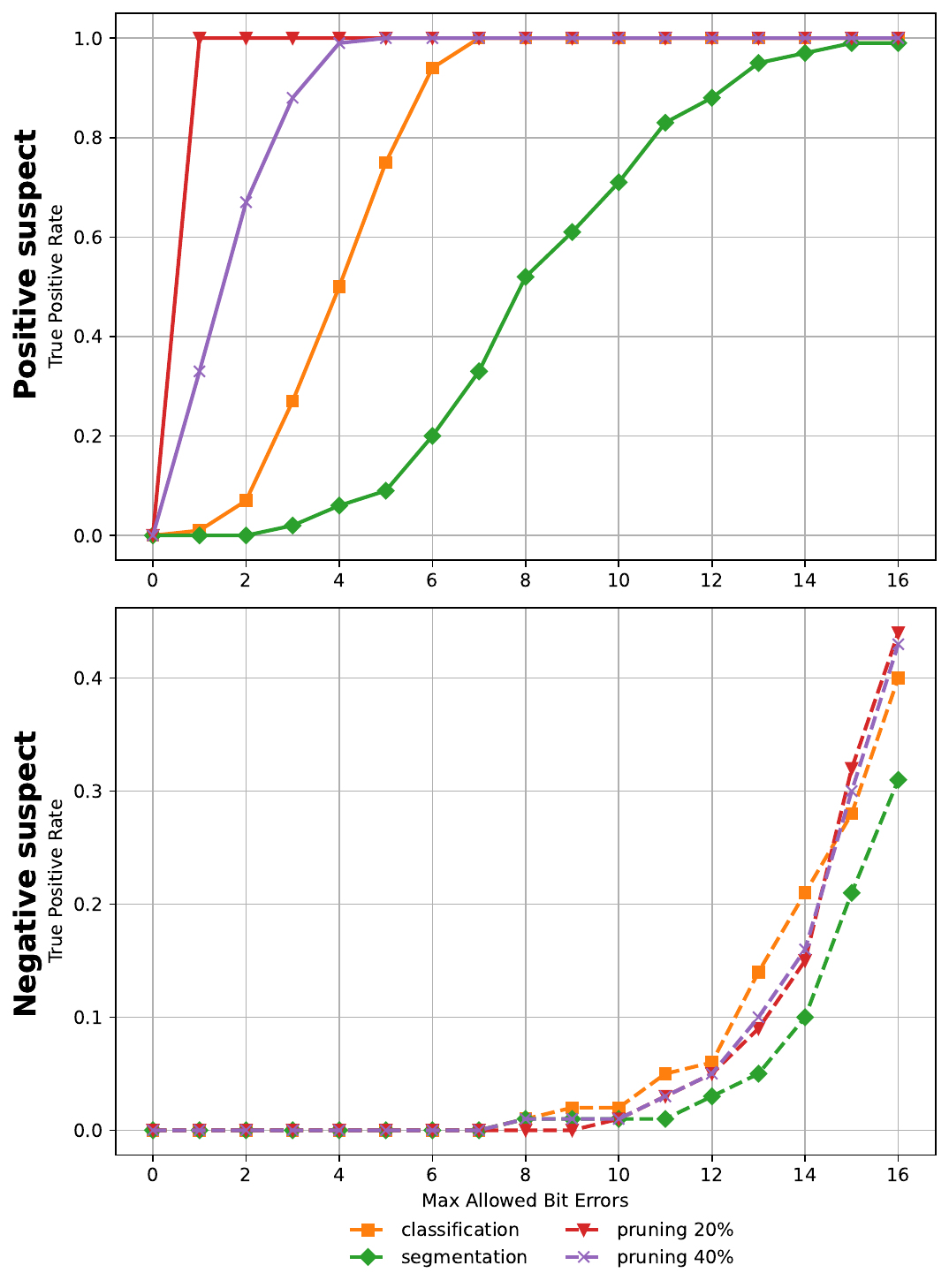} }}%
    \caption{Watermark detection rate $R$ from \eqref{eq:R}, averaged over $N=1000$ images used for watermarking. Classification experiments were conducted on the E-commerce Product Images dataset,  segmentation experiment was conducted on the FoodSeg103 dataset. }%
    \label{fig:results}%
\end{figure}

\begin{figure}[t]
\centering
\begin{subfigure}{0.8\linewidth}  % можно регулировать ширину
    \centering
    \includegraphics[width=\linewidth]{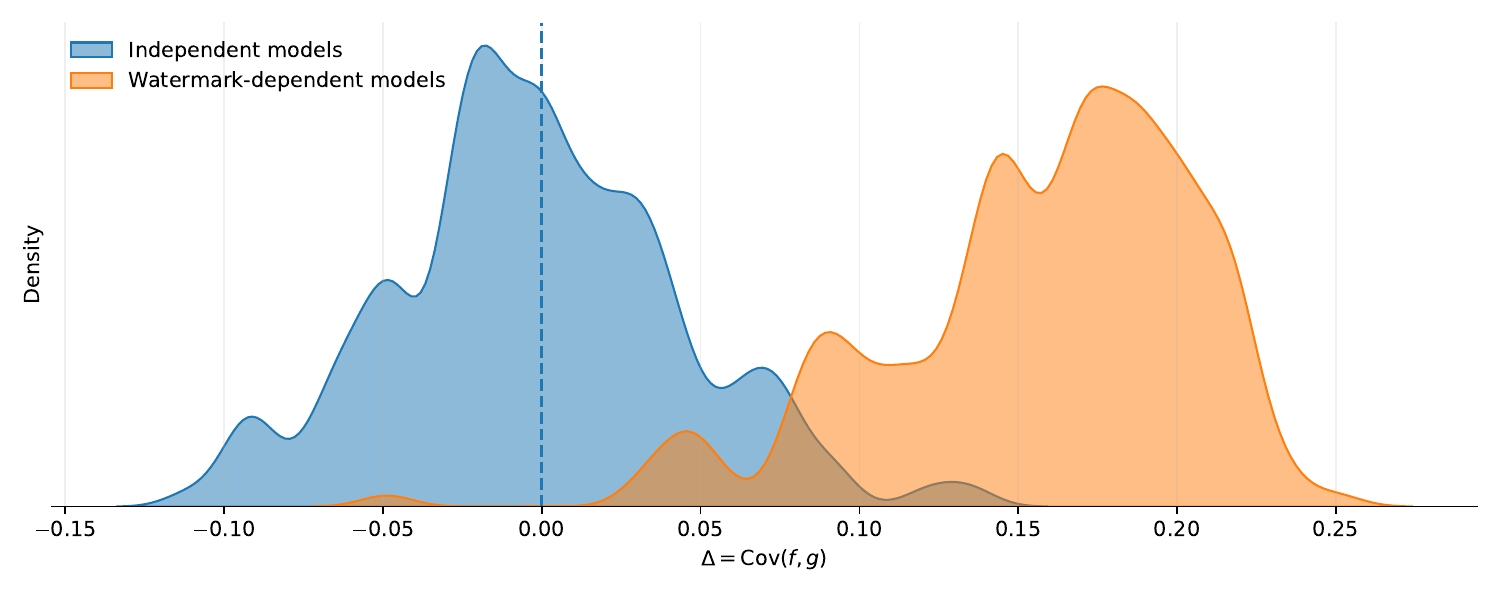}
    \caption{The architecture of VFM is DINOv2}
\end{subfigure}

\vspace{0.5cm}  % вертикальный отступ между картинками

\begin{subfigure}{0.8\linewidth}
    \centering
    \includegraphics[width=\linewidth]{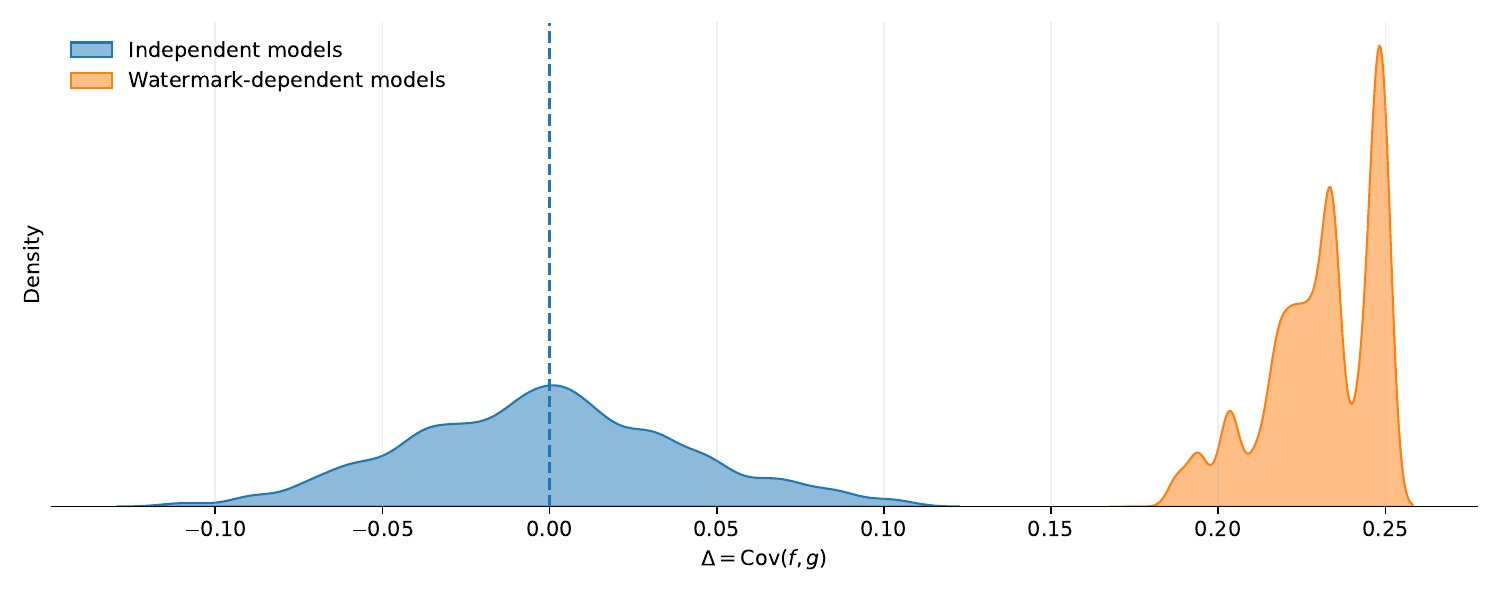}
    \caption{The architecture of VFM is CLIP}
\end{subfigure}

\caption{Distribution of covariance between decoded watermark messages from two models $f$ and $g$. Independent models produce covariance values near zero, while watermark-dependent models exhibit positive covariance due to correlated decoding of watermark bits.}
\label{fig:covariance_distribution}
\end{figure}

Experimentally, we assess the efficiency of RandMark \ by computing the average watermark detection rate from \eqref{eq:R} for different values of maximum number of bit errors, $\tau.$ In Fig. \ref{fig:results}, we demonstrate that our method can be used to reliably detect models that are functionally connected to the watermarked one, namely, the ones obtained via fine-tuning for downstream tasks and pruning. At the same time, RandMark \  does not falsely detect the presence of the watermarks in negative suspect models.

In addition to the detection rate, we analyze correlations between decoded watermark messages. Let $f$ and $g$ denote two models, and let $m'
(f)$ and $m'(g)$ be the watermark bit sequences decoded from them. The covariance between the decoded messages is estimated as
\[
\Delta = \frac{\mathbb{V}(m'(f)) + \mathbb{V}(m'
(g)) - \mathbb{V}(m'(f) - m'(g))}{2} = \text{cov}(f, g).
\]
In Fig.~\ref{fig:covariance_distribution} for independent models, $\Delta$ is close to zero, indicating negligible correlation. In contrast, watermark-dependent models exhibit positive covariance, reflecting correlated decoding of watermark bits. This complementary metric provides additional evidence of functional dependence between models and can help distinguish watermarked models from unrelated ones.

\subsection{Comparison with the baseline fingerprinting and watermarking approaches}

\begin{table}[t]
\centering
\caption{Comparison with baseline watermarking methods under segmentation fine-tuning. 
Segmentation performance and watermark extraction accuracy are reported after different numbers of fine-tuning epochs; the architecture of the source VFM is CLIP.}
\begin{tabular}{lcccccc}
\toprule
& \multicolumn{2}{c}{1 epoch} & \multicolumn{2}{c}{3 epochs} & \multicolumn{2}{c}{5 epochs} \\
\cmidrule(lr){2-3} \cmidrule(lr){4-5} \cmidrule(lr){6-7}
Method & Segm. $\uparrow$ & WM $\uparrow$ & Segm. $\uparrow$ & WM $\uparrow$ & Segm. $\uparrow$ & WM $\uparrow$ \\
\midrule
Randomized Smoothing & 0.14  & 0.27 & 0.36 & 0.00 & 0.46 & 0.00 \\
Ours & \textbf{0.32} & \textbf{1.00} & \textbf{0.52} & \textbf{0.99} & \textbf{0.55} & \textbf{0.97} \\
\bottomrule
\end{tabular}
\end{table}

We indicate the lack of fingerprinting methods designed specifically for visual foundation models. To compare our approach with some of the general-purpose fingerprinting approaches, we add a fine-tuned classification head to the source VFM. The classification head concatenates the CLS token with the mean of patch tokens, applies normalization and dropout, and feeds the result into a linear classifier. Here, the classification backbone  is  fine-tuned on the ImageNet dataset. Thus, we compare fingerprinting approaches in a classification scenario. Specifically, we compare RandMark \ with ADV-TRA \cite{xu2024united} and IPGuard \cite{cao2021ipguard} and report results in Table \ref{tab:comparison}, where we present average watermark detection rates  for both positive and negative suspect models. Specifically, negative suspect models here are the ones of different architecture (DINOv2 with registers and CLIP). There, experiments with the positive suspect models correspond to the ones reported in Fig. \ref{fig:results}. It is noteworthy that the proposed method outperforms general-purpose fingerprinting techniques in terms of watermark detection rate, both for positive and negative suspect models. 

To evaluate the robustness of the proposed method, we compare it with the baseline approach proposed in~\cite{bansal22a, ren23c}. In particular, we fine-tune the watermarked VFMs on an image segmentation task and monitor both the downstream task performance and the ability to recover the embedded watermark. This evaluation allows us to analyze whether the fingerprint remains detectable after task adaptation and whether the downstream performance is preserved.

The results are presented in Table~\ref{tab:comparison}. We observe that the baseline approach introduces a substantial degradation of the downstream task performance while also failing to preserve the watermark under fine-tuning. In contrast, our method maintains high watermark extraction accuracy while achieving significantly better segmentation performance, indicating that the proposed approach preserves both task adaptability and fingerprint detectability.

\begin{table}[htb]
    \centering
    \caption{Quantitative comparison with the general-purpose fingerprinting methods. We report the average watermark detection rate; the architecture of the source VFM is DINOv2.}
    \begin{tabular}{ccccc}
    \toprule
    Model type & Experiment &  RandMark & ADV-TRA & IPGuard \\
    \midrule
    \multirow{4}{*}{Positive suspect $\uparrow$} & 
     classification & $0.870$ & $0.021$ & $0.000$\\
    & segmentation & $0.750$ & $0.000$ & $0.000$\\
    & pruning (20\%) & $1.000$ & $1.000$ & $0.530$\\
     & pruning (40\%) & $1.000$ & $0.086$ & $0.000$\\
     \midrule
     \multirow{2}{*}{Negative suspect $\downarrow$} & DINOv2 w/ registers & $0.000$ & $0.012$ & $0.010$ \\
     & CLIP & $0.000$ & $0.000$ & $0.000$ \\
    \bottomrule
    \end{tabular}
    \label{tab:comparison}
\end{table}

% To assess the computational complexity of the watermarking methods, in Table \ref{tab:time} we present the time in minutes required to embed and extract the watermarks. All the experiments were conducted on a single GPU Nvidia Tesla A100 80GB. 

% \begin{table}[htb]
%     \centering
%     \caption{Comparison of computation complexity of the watermarking methods. For ActiveMark, we report cumulative results for  $N=1000$ images used for watermarking.}
%     \begin{tabular}{ccc}
%     \toprule
%     Method & Watermark embedding & Watermark extraction \\ 
%     \midrule
%     ActiveMark & $34.63 $ & $4.02$ \\
%     ADV-TRA & $1663.70$ & $3.41$\\ 
%     IPGuard & $1868.54$ & $11.62$\\ 
%     \bottomrule
%     \end{tabular}
%     \label{tab:time}
% \end{table}
To assess the computational complexity of the watermarking and fingerprinting methods, we refer the reader to the Supplementary Material.

%Specifically, we use a lightweight encoder-decoder architecture 



\section{Conclusion}
In this work, we propose RandMark, a novel watermarking approach for visual foundation models. This method is model agnostic:  it is worth mentioning that the model’s owner has to prepare a set of input images and perform the watermark embedding procedure only once for a given instance of the model; then, the watermarked model remains detectable by our method after fine-tuning to a particular downstream task (for example, image classification and segmentation). % our experiments show that the proposed method allows to reliably detect the functional copies of a particular foundation model obtained by the fine-tuning and pruning both for image classification and segmentation. 
On the other hand, we verify that RandMark \  does not detect benign, independent models as functional copies of the watermarked VFM, which makes the method applicable in practical scenarios. We theoretically show that our method, by design, yields low false positive and false negative detection rates.

\bibliographystyle{splncs04}
\bibliography{main}
\end{document}